\PassOptionsToPackage{unicode}{hyperref}
\PassOptionsToPackage{hyphens}{url}
\PassOptionsToPackage{dvipsnames,svgnames,x11names}{xcolor}
\documentclass[
  12pt,
  letterpaper,
]{article}

\usepackage{amsmath,amssymb}
\usepackage{setspace}
\usepackage{iftex}
\ifPDFTeX
  \usepackage[T1]{fontenc}
  \usepackage[utf8]{inputenc}
  \usepackage{textcomp} 
\else 
  \usepackage{unicode-math}
  \defaultfontfeatures{Scale=MatchLowercase}
  \defaultfontfeatures[\rmfamily]{Ligatures=TeX,Scale=1}
\fi
\usepackage{lmodern}
\ifPDFTeX\else  
\fi
\IfFileExists{upquote.sty}{\usepackage{upquote}}{}
\IfFileExists{microtype.sty}{
  \usepackage[]{microtype}
  \UseMicrotypeSet[protrusion]{basicmath} 
}{}
\makeatletter
\@ifundefined{KOMAClassName}{
  \IfFileExists{parskip.sty}{%
    \usepackage{parskip}
  }{
    \setlength{\parindent}{0pt}
    \setlength{\parskip}{6pt plus 2pt minus 1pt}}
}{
  \KOMAoptions{parskip=half}}
\makeatother
\usepackage{xcolor}
\usepackage[top=0.75in,bottom=0.75in,left=1in,right=1in]{geometry}
\makeatletter
\ifx\paragraph\undefined\else
  \let\oldparagraph\paragraph
  \renewcommand{\paragraph}{
    \@ifstar
      \xxxParagraphStar
      \xxxParagraphNoStar
  }
  \newcommand{\xxxParagraphStar}[1]{\oldparagraph*{#1}\mbox{}}
  \newcommand{\xxxParagraphNoStar}[1]{\oldparagraph{#1}\mbox{}}
\fi
\ifx\subparagraph\undefined\else
  \let\oldsubparagraph\subparagraph
  \renewcommand{\subparagraph}{
    \@ifstar
      \xxxSubParagraphStar
      \xxxSubParagraphNoStar
  }
  \newcommand{\xxxSubParagraphStar}[1]{\oldsubparagraph*{#1}\mbox{}}
  \newcommand{\xxxSubParagraphNoStar}[1]{\oldsubparagraph{#1}\mbox{}}
\fi
\makeatother

\usepackage{longtable,booktabs,array}
\usepackage{calc} 
\usepackage{etoolbox}
\makeatletter
\patchcmd\longtable{\par}{\if@noskipsec\mbox{}\fi\par}{}{}
\makeatother
\IfFileExists{footnotehyper.sty}{\usepackage{footnotehyper}}{\usepackage{footnote}}
\makesavenoteenv{longtable}
\usepackage{graphicx}
\makeatletter
\newsavebox\pandoc@box
\newcommand*\pandocbounded[1]{
  \sbox\pandoc@box{#1}%
  \Gscale@div\@tempa{\textheight}{\dimexpr\ht\pandoc@box+\dp\pandoc@box\relax}%
  \Gscale@div\@tempb{\linewidth}{\wd\pandoc@box}%
  \ifdim\@tempb\p@<\@tempa\p@\let\@tempa\@tempb\fi
  \ifdim\@tempa\p@<\p@\scalebox{\@tempa}{\usebox\pandoc@box}%
  \else\usebox{\pandoc@box}%
  \fi%
}
\def\fps@figure{htbp}
\makeatother

\usepackage{algorithm}
\usepackage{algpseudocode}
\usepackage{graphicx}
\usepackage{scrlayer-scrpage}
\usepackage{etoolbox}
\usepackage{setspace}
\clearpairofpagestyles
\ofoot*{\pagemark}
\AtBeginDocument{}
\makeatletter
\@ifpackageloaded{caption}{}{\usepackage{caption}}
\AtBeginDocument{%
\ifdefined\contentsname
  \renewcommand*\contentsname{Table of contents}
\else
  \newcommand\contentsname{Table of contents}
\fi
\ifdefined\listfigurename
  \renewcommand*\listfigurename{List of Figures}
\else
  \newcommand\listfigurename{List of Figures}
\fi
\ifdefined\listtablename
  \renewcommand*\listtablename{List of Tables}
\else
  \newcommand\listtablename{List of Tables}
\fi
\ifdefined\figurename
  \renewcommand*\figurename{Figure}
\else
  \newcommand\figurename{Figure}
\fi
\ifdefined\tablename
  \renewcommand*\tablename{Table}
\else
  \newcommand\tablename{Table}
\fi
}
\@ifpackageloaded{float}{}{\usepackage{float}}
\floatstyle{ruled}
\@ifundefined{c@chapter}{\newfloat{codelisting}{h}{lop}}{\newfloat{codelisting}{h}{lop}[chapter]}
\floatname{codelisting}{Listing}

\usepackage{amsthm}
\theoremstyle{plain}
\newtheorem{proposition}{Proposition}[section]
\theoremstyle{plain}
\newtheorem{lemma}{Lemma}[section]
\theoremstyle{plain}
\newtheorem{theorem}{Theorem}[section]
\theoremstyle{remark}
\AtBeginDocument{}

\newtheorem{refremark}{Remark}[section]

\makeatother
\makeatletter
\@ifpackageloaded{caption}{}{\usepackage{caption}}
\@ifpackageloaded{subcaption}{}{\usepackage{subcaption}}
\makeatother

\usepackage[]{natbib}
\usepackage{bookmark}

\IfFileExists{xurl.sty}{\usepackage{xurl}}{} 
\hypersetup{
  colorlinks=true,
  linkcolor={blue},
  filecolor={Maroon},
  citecolor={blue},
  urlcolor={Blue},
  pdfcreator={LaTeX via pandoc}}

\author{}
\date{}

\begin{document}

\setstretch{1}
\title{Feedback-Coupled Memory Systems in Continuous Time}

\author{
Stefano Grassi\textsuperscript{a}\thanks{Corresponding author: stefano.g@bu.ac.th}
}

\date{
\textsuperscript{a}Bangkok University, Phahonyothin Rd, Khlong Nueng, Khlong Luang District, Pathum Thani 12120, Thailand\\[2ex]
June 24, 2026
}

\maketitle

\begin{abstract}
The Feedback-Coupled Memory Systems (FCMS) architecture formalizes
closed-loop coordination through four abstract operators, two of which
--- the agent update operator $f_i$ and the environmental update operator
$\Psi$ --- are left axiomatically undefined in the original framework. To
address this, $f_i$ is defined by Mechanism-Based Intelligence (MBI),
where agents update locally through a decentralized price mechanism and
economic principles, and $\Psi$ is defined by the Coupled Memory Graph
Process (CMGP), a non-Markovian framework where the environment is
treated as a physical substrate that records and responds to trajectory
history coherently without external forcing. The resulting
continuous-time FCMS instantiation achieves Lyapunov global
dissipativity governed by the computable threshold $4\beta^2 < 2\eta\mu\gamma^2$. 
This generalizes both the discrete FCMS stability
condition $4\eta\beta^2 < \gamma$ and CMGP's physical bifurcation
threshold $\alpha_c = 1/K$, confirming that memory dissipation must
outpace feedback gain as a universal organizing principle. Numerical
simulation with $N=2$ agents and mean-field validation at $N=10^6$
confirm the stability threshold and the self-reinforcing coordination
cascade that emerges when it is violated.
\end{abstract}

\noindent\textbf{Keywords:} feedback-coupled memory systems,
continuous-time coordination, Lyapunov dissipativity, Hopf bifurcation,
mechanism design, graph Laplacian dynamics, memory engine, early warning
signals \vspace{1em}

\section{Introduction}\label{sec-introduction}

Distributed agents interacting through persistent environments must
achieve collective order without centralized control, a challenge that
lies at the intersection of economics, dynamical systems theory, and
multi-agent artificial intelligence \citep{hayek1945}. This work extends
FCMS \citep{grassi2026fcms} to continuous time by providing explicit
functional forms for its two axiomatically defined operators \(f_i\) and
\(\Psi\). FCMS is a closed-loop dynamical architecture formalizing this
feedback loop through four operators \(\mathcal{A}\), \(\Phi\), \(f_i\),
\(\Psi\), reviewed in Section~\ref{sec-theoretical-background}. FCMS
proves that under dissipativity the system admits a bounded
forward-invariant region, coordination cannot be reduced to static
optimization, and bidirectional coupling is necessary. However, the
operators \(f_i\) and \(\Psi\) are defined axiomatically in FCMS leaving
their continuous-time instantiation an open problem. To close this gap,
this paper proposes to instantiate them with two frameworks from the
literature: Mechanism-Based Intelligence (MBI) \citep{grassi2025mbi} and
the Coupled Memory Graph Process (CMGP)
\citep{sarkar2025cmgp, sarkar2025memory}. MBI is a mechanism-design
framework where the Differentiable Price Mechanism (DPM) computes
incentives as a Vickrey--Clarke--Groves (VCG)-equivalent signal
\citep{vickrey1961, clarke1971, groves1973} guaranteeing Dominant
Strategy Incentive Compatibility (DSIC) and convergence
\citep{hurwicz2006}. MBI's discrete gradient update
\(\mathbf{x}_{i,t+1}=\mathbf{x}_{i,t}+\eta G_{i,t}\), where
\(G_i=-\nabla_{\mathbf{x}_i}\mathcal{L}^{\mathrm{global}}\) is the DPM
incentive signal, is the natural candidate for the continuous-time
instantiation of \(f_i\), becoming the gradient flow
\(\dot{\mathbf{x}}_i=-\eta\nabla_{\mathbf{x}_i}\mathcal{L}_i+\xi_i\) in
the continuous limit. CMGP is a physics graph framework where
memory-driven feedback generates coherence without external forcing
\citep{sarkar2025cmgp, sarkar2025memory}. CMGP's continuous memory field
\(\partial_t S(\mathbf{r},t)=-\alpha_s S+A\int_0^t
\Theta_s(t-\tau)G_\sigma(\mathbf{r}-\mathbf{r}(\tau))\,d\tau\),
governing how a physical substrate records and responds to trajectory
history, is the natural candidate for the continuous-time instantiation
of \(\Psi\), translated to the discrete multi-agent graph setting as
evolving edge weights
\(\dot{w}_{ij}=-\gamma w_{ij}+\psi(\mathbf{x}_i,\mathbf{x}_j)
\mathbb{I}_{ij}\). Taken together, the continuous-time FCMS yields a
fully specified closed-loop system whose global stability is governed by
the computable threshold \(4\beta^2 < 2\eta\mu\gamma^2\). This condition
generalizes both the discrete FCMS stability condition
\(4\eta\beta^2 < \gamma\) \citep{grassi2026fcms} and CMGP's physical
bifurcation threshold \(\alpha_c = 1/K\) \citep{sarkar2025memory},
confirming that memory dissipation must outpace feedback gain as a
universal organizing principle across physical and strategic multi-agent
systems. The paper is organized as follows.
Section~\ref{sec-theoretical-background} reviews MBI, CMGP, and FCMS and
identifies the open gaps in each.
Section~\ref{sec-theoretical-framework} develops the continuous-time
system by instantiating \(f_i\) and \(\Psi\).
Section~\ref{sec-stability-analysis} proves global dissipativity and
derives the stability threshold. Section~\ref{sec-simulation} validates
the threshold numerically with \(N=2\) agents.
Section~\ref{sec-discussion} discusses limitations and future
directions. Section~\ref{sec-conclusion} concludes and states the
implications for coordination theory.

\section{Theoretical Background}\label{sec-theoretical-background}

\subsection{MBI}\label{sec-mbi}

MBI \citep{grassi2025mbi} is characterized by a Differentiable Directed
Acyclic Graph (D-DAG) mapping interacting rational self-interested
agents \(A_i\) to a Planner \(P\), the global entity that defines the
system's overarching objective \(\mathcal{L}^{\text{global}}\). The
Planner is defined externally and holds no private information; its sole
role is institutional design \citep{hurwicz2006}. To ensure each agent's
optimal action \(\mathbf{x}_i^*\) aligns with the global objective, the
Planner leverages the forward and backward pass of the D-DAG to compute
the incentive signal

\begin{equation}\phantomsection\label{eq-incentive-signal}{\mathbf{G}_i = -\nabla_{\mathbf{x}_i}\mathcal{L}^{\text{global}}
\tag{1}}\end{equation}

which represents the negative marginal externality of agent action
\(\mathbf{x}_i\) on the global loss
\citep{vickrey1961, clarke1971, groves1973}. The framework is formalized
by four theorems: (i) DSIC, truth-telling is each agent's dominant
strategy regardless of others' actions; (ii) BR, the optimal stopping
condition for computational effort; (iii) BIC, truth-telling remains
optimal in expectation under asymmetric information about agents'
private types \citep{myerson1981}; (iv) Global Convergence, joint
self-interested optimization converges to the unique global optimum
under strict convexity and Lipschitz continuity of
\(\mathcal{L}^{\text{global}}\) \citep[Appendix B.6]{grassi2025mbi}. MBI
operates in discrete iterations over a D-DAG. The continuous-time limit
of the agent update operator \(f_i\) and its behavior under persistent
environmental feedback is not established in the original framework.

\subsection{CMGP}\label{sec-cmgp}

CMGP \citep{sarkar2025cmgp} is a non-Markovian framework where
interacting agents, represented as nodes on a directed graph, achieve
persistent coherence through heterogeneous memory and asymmetric
coupling. This autonomous, closed-loop dynamics generate structured,
phase-locked motion without external forcing. Considering a single
memoryless Brownian particle \citep{sarkar2025memory}, define
\(\mathbf{r}(t)\in
\mathbb{R}^2\) as the position of the particle and \(S:\mathbb{R}^2
\times\mathbb{R}_{+}\rightarrow\mathbb{R}\) as the scalar memory field.
The scalar memory field evolves dynamically according to the
integro-differential equation

\begin{equation}\phantomsection\label{eq-memory-field}{\partial_t S(\mathbf{r},t)=-\alpha_s S(\mathbf{r},t)+A\int_{0}^{t}
\Theta_s(t-\tau)\,G_{\sigma}\!\left(\mathbf{r}-\mathbf{r}(\tau)\right)
\,d\tau \tag{2}}\end{equation}

where \(S\) naturally dissipates through the decay term \(-\alpha_s
S(\mathbf{r},t)\) at rate \(\alpha_s=1/\tau_s\), while new information
is deposited with strength \(A\). The temporal kernel
\(\Theta_s(t-\tau)=
\alpha_s e^{-\alpha_s(t-\tau)}\) exponentially downweights older
history, giving more weight to recent trajectory history. Spatially,
\(G_{\sigma}
\!\left(\mathbf{r}-\mathbf{r}(\tau)\right)\) distributes deposited
memory over a characteristic width \(\sigma\) around the particle's
historical position \(\mathbf{r}(\tau)\). The particle's motion is
governed by a stochastic Volterra equation coupling intrinsic velocity
memory \(\mathcal{K}_m\) to the gradient feedback force
\(\kappa\nabla S\):

\begin{equation}\phantomsection\label{eq-particle-motion}{\dot{\mathbf{r}}(t)=\int_0^t \mathcal{K}_m(t-\tau)\dot{\mathbf{r}}
(\tau)d\tau-\kappa\nabla S(\mathbf{r}(t),t)+\xi(t) \tag{3}}\end{equation}

The memory field \(S(\mathbf{r},t)\) shapes the particle's trajectory
via \(\nabla S\), while the particle continuously rewrites \(S\) through
its motion. This reciprocal coupling gives rise to what Sarkar terms a
\emph{memory engine}: a self-organizing mechanism converting stored
trajectory history into predictive motion without external tuning
\citep{sarkar2025memory}. The stability of this coherence is governed by
a bifurcation threshold \(\alpha_s=\alpha_c:=\frac{1}{K}\), where \(K\)
encodes the transverse curvature of the memory field. Above this
threshold, perturbations decay and the current trajectory persists;
below it, feedback amplifies deviations, destabilizing uniform motion
and giving rise to phase-locked regimes. Coherence emerges at a
nonequilibrium fixed point defined by
\(\dot{\varepsilon}_s = I(t) - D(t) \approx 0\), where \(I(t)\) is the
energy injected into the field by the particle's imprinting and \(D(t)\)
is the energy dissipated through field decay. While Sarkar's framework
establishes memory-driven coherence for a single physical particle, its
extension to \(N\) strategic agents whose self-interested optimization
requires explicit incentive alignment, and its role as the environmental
update operator \(\Psi\) within a multi-layer coordination system,
remains open.

\subsection{FCMS}\label{sec-fcms}

FCMS \citep{grassi2026fcms} is a dynamical framework where the joint
configuration of agent states \(\mathbf{x}_t\) and a persistent
environment \(S_t\) projects a global coordination signal
\(L_t^{\text{global}}\) which is distributed through an incentive field
\(\mathbf{G}_t\), updating agent states \(\mathbf{x}_{t+1}\) and the
persistent environment \(S_{t+1}\) recursively as follows:

\begin{equation}\phantomsection\label{eq-fcms-chain}{(\mathbf{x}_t,S_t)\to L_t^{\text{global}}\to \mathbf{G}_t\to
(\mathbf{x}_{t+1},S_{t+1}). \tag{4}}\end{equation}

The closed-loop dynamics is described by four operators

\begin{equation}\phantomsection\label{eq-fcms-operators}{\mathbf{x}_{t+1}=f_i(\mathbf{x}_t,\mathbf{G}_t,S_t), \quad
S_{t+1}=\Psi(S_t,\mathbf{x}_t), \quad L_t^{\text{global}}=
\mathcal{A}(\mathbf{x}_t,S_t), \quad \mathbf{G}_t=\Phi(L_t^{\text{global}},
\mathbf{x}_t,S_t) \tag{5}}\end{equation}

where \(f_i\) maps current agent states, incentive signals and
environmental state into future ones; \(\Psi\) accumulates agent
activity and environmental state into the next persistent state;
\(\mathcal{A}\) projects the joint configuration \((\mathbf{x}_t,S_t)\)
onto the scalar global coordination signal \(L_t^{\text{global}}\); and
the incentive distribution operator
\(\Phi: \mathbb{R} \times \mathcal{X} \times
\mathcal{S} \rightarrow \mathbb{R}^{N}\) translates the global
coordination signal \(L_t^{\text{global}}\) into local directional
pressures experienced by each agent. Following FCMS \citep[Section
2.3]{grassi2026fcms}, \(\Phi\) must distribute the global signal
locally, must be non-conservative with respect to \(\mathcal{X}\), and
must be continuous. This operator is instantiated by the minimal linear
specification \(\Phi(S_t,L_t^{\text{global}},
\mathbf{x}_i)=\alpha_1 S_t\mathbf{x}_i+\alpha_2\mathbf{L}_t\mathbf{x}_i\),
where \(\mathbf{L}_t\) is the network topology defined in
Section~\ref{sec-network-layer} and \(\alpha_1,\alpha_2>0\) are coupling
parameters. The system yields four structural results: (i) under
dissipativity there exists a bounded forward-invariant set \citep[Prop.
A.1.1]{grassi2026fcms}; (ii) memory-dependent incentives generically
prevent reduction to a static optimization over agent states
\citep[Prop. A.2.1]{grassi2026fcms}; (iii) persistent environmental
memory transmits initial differences forward in time, inducing history
sensitivity \citep[Prop. A.3.1]{grassi2026fcms}; (iv) bidirectional
coupling between incentives and environmental memory is a necessary
condition for adaptive coordination \citep[Prop. A.4.1]{grassi2026fcms}.
The system remains locally asymptotically stable if and only if

\begin{equation}\phantomsection\label{eq-fcms-discrete-stability}{4\eta\beta^2<\gamma \tag{6}}\end{equation}

above which the system undergoes a Neimark-Sacker bifurcation
\citep{kuznetsov2004}. While the framework establishes the architectural
backbone of closed-loop coordination, the operators \(f_i\) and \(\Psi\)
are defined axiomatically; their continuous-time instantiation
connecting agent optimization and network topology has yet to be
explored.

\section{Theoretical Framework}\label{sec-theoretical-framework}

\subsection{System Configuration}\label{sec-system-configuration}

\subsubsection{Agent Layer}\label{sec-agent-layer}

Define the agent state vector \(\mathbf{x}_i(t) \in \mathcal{X}_i
\subseteq \mathbb{R}^{d_i}\) in continuous time and \(\mathbf{X}(t) =
(\mathbf{x}_1,\dots,\mathbf{x}_N)^T \in \mathbb{R}^{N\times d}\) the
collective state, where \(d = \sum_i d_{i}\).

\subsubsection{Network Layer}\label{sec-network-layer}

Define the time-varying graph \(\mathcal{G}_t = (\mathcal{V},
\mathcal{E}_t, \mathbf{W}_t)\) \citep{newman2010, mesbahi2010}, where
the node set \(\mathcal{V} = \{1,\dots,N\}\) represents the agent set,
\(\mathcal{E}_t \subseteq \mathcal{V} \times \mathcal{V}\) the active
edge set at time \(t\) and \(\mathbf{W}_t \in \mathbb{R}^{N\times N}\)
the adjacency weight matrix with \(w_{ij}(t) \geq 0\), the edge weight
between agents \(i\) and \(j\). Define
\(\mathbf{D}_t \in \mathbb{R}^{N\times N}\) the diagonal degree matrix,
where \(\mathbf{D}_{ii} = \sum_j w_{ij}\) and
\(\mathbf{L}_t \in \mathbb{R}^{N\times N}\) the graph Laplacian
\(\mathbf{D}_t - \mathbf{W}_t\) \citep{olfati2007}.

\subsubsection{Environment Layer}\label{sec-environment-layer}

Define \(S_t\in\mathbb{R}\) as the environmental memory state, retained
as a scalar consistent with the FCMS linear specification
\citep[Appendix A.5]{grassi2026fcms}. The general formulation admits
\(S_t\in\mathcal{S}\subseteq\mathbb{R}^m\); the scalar case is used here
for analytical tractability.

\subsubsection{Full System State}\label{sec-full-system-state}

Merging the layers, the full system state is defined as
\((\mathbf{X}(t), \mathbf{W}_t, S_t)\).

\subsection{Agent Continuous-Time
Lifting}\label{sec-agent-continuous-time-lifting}

To instantiate \(f_i\), substitute the MBI update \(\mathbf{x}_{i,t+1} =
\mathbf{x}_{i,t} + \eta G_{i,t}\) into the FCMS discrete map
\(T(\mathbf{x}_t,S_t)\), which yields in the limit \(\Delta t\to 0\),

\begin{equation}\phantomsection\label{eq-agent-ode}{\dot{\mathbf{x}}_i = -\eta_i \nabla_{\mathbf{x}_i} L_i\!\left(
\mathbf{x}_i,\mathbf{X}_{-i}; \Phi(S_t,\mathbf{L}_t,\mathbf{x}_i)
\right) + \xi_i. \tag{7}}\end{equation}

This is the continuous-time limit of the MBI Theorem 4 convergence
dynamics \citep[Appendix B.6]{grassi2025mbi}, where the discrete
gradient update becomes a continuous gradient flow. Here, \(\eta_i>0\)
is the agent learning rate governing the speed of gradient descent,
\(\nabla_{\mathbf{x}_i}\mathcal{L}_i(\mathbf{x}_i,\mathbf{X}_{-i};\Phi)\)
is the gradient of agent \(i\)'s local loss with respect to its own
state \(\mathbf{x}_i\), holding the states of all other agents
\(\mathbf{X}_{-i}\) fixed, and evaluated under the incentive field
\(\Phi(S_t,\mathbf{L}_t,\mathbf{x}_i)=\alpha_1 S_t\mathbf{x}_i+
\alpha_2\mathbf{L}_t\mathbf{x}_i\). This coupling is what distinguishes
Equation~\ref{eq-agent-ode} from a standard gradient flow: the loss
landscape itself evolves through the feedback loop. This follows from
substituting the discrete update rule:

\[\mathbf{x}_{i,t+1} = \mathbf{x}_{i,t} - \eta_i \nabla_{\mathbf{x}_i}
L_i\!\left(\mathbf{x}_{i,t},\mathbf{X}_{-i,t};
\Phi(S_t,\mathbf{L}_t,\mathbf{x}_{i,t})\right)\Delta t + \xi_i \Delta t,\]

which implies

\[\frac{\mathbf{x}_{i,t+1}-\mathbf{x}_{i,t}}{\Delta t} =
-\eta_i \nabla_{\mathbf{x}_i} L_i\!\left(\mathbf{x}_{i,t},
\mathbf{X}_{-i,t};\Phi(S_t,\mathbf{L}_t,\mathbf{x}_{i,t})\right)
+ \xi_i.\]

Taking the limit \(\Delta t\to 0\) yields

\[\lim_{\Delta t\to 0}\frac{\mathbf{x}_{i,t+1}-\mathbf{x}_{i,t}}
{\Delta t} = \dot{\mathbf{x}}_i,\]

and therefore

\[\dot{\mathbf{x}}_i = -\eta_i \nabla_{\mathbf{x}_i} L_i\!\left(
\mathbf{x}_i,\mathbf{X}_{-i};\Phi(S_t,\mathbf{L}_t,\mathbf{x}_i)
\right) + \xi_i.\]

\subsection{Network Layer Insertion}\label{sec-network-layer-insertion}

Define \(\Psi\) with CMGP's continuous spatial imprinting dynamics
(Equation~\ref{eq-memory-field}), which in the discrete multi-agent
graph setting becomes

\begin{equation}\phantomsection\label{eq-network-ode}{\dot{w}_{ij} = -\gamma w_{ij} + \psi(\mathbf{x}_i,\mathbf{x}_j)
\mathbb{I}_{ij}. \tag{8}}\end{equation}

Here, CMGP's \(\alpha_s\) becomes the edge memory decay rate \(\gamma\),
\(G_{\sigma}\) becomes the Gaussian edge kernel
\(\psi(\mathbf{x}_i,\mathbf{x}_j):=\exp\!\left(-\frac{\|\mathbf{x}_i-
\mathbf{x}_j\|^2}{2\sigma^2}\right)\), the continuous memory field
\(S(\mathbf{r},t)\) becomes the discrete weighted interaction graph
\(\mathbf{W}_t=[w_{ij}]\), and CMGP's single particle trajectory
\(\mathbf{r}(t)\) becomes the population of \(N\) strategic agents
\(\{\mathbf{x}_1,\ldots,\mathbf{x}_N\}\); \(\mathbb{I}_{ij}\in\{0,1\}\)
is the edge indicator equal to \(1\) if \((i,j)\in\mathcal{E}_t\).
Unlike CMGP's single physical particle, which moves under gradient
forcing alone \citep{sarkar2025memory}, each agent \(i\) in the present
framework follows an incentive-aligned gradient flow governed by the
DPM, requiring MBI to ensure strategic coherence across the population.
Equation~\ref{eq-network-ode} may therefore be interpreted as the
graph-theoretic discretization of Equation~\ref{eq-memory-field}, where
memory is no longer stored in a continuous spatial field but in evolving
edge weights encoding pairwise interaction history \citep{olfati2007}.
The FCMS environmental accumulation operator \(\Psi\) then closes the
feedback loop through

\begin{equation}\phantomsection\label{eq-env-ode}{\dot{S}_t = -\mu S_t + \beta\sum_{i,j}w_{ij}\,\operatorname{tr}\!
\left(\mathbf{x}_i\mathbf{x}_j^{\top}\right) \tag{9}}\end{equation}

where the environment \(S_t\) accumulates weighted interaction structure
from the evolving graph; the trace \(\operatorname{tr}(\mathbf{x}_i
\mathbf{x}_j^T)\) projects the pairwise outer product onto \(S_t\),
thereby providing the persistent memory state that feeds back into agent
incentives and future coordination dynamics.

\subsection{Complete Coupled System}\label{sec-complete-coupled-system}

Collecting the three layer instantiations, the complete continuous-time
realization of FCMS is governed by the following coupled system of
ordinary differential equations:

\[\dot{\mathbf{x}}_i = -\eta_i \nabla_{\mathbf{x}_i} L_i\!\left(
\mathbf{x}_i,\mathbf{X}_{-i};\Phi(S_t,\mathbf{L}_t,\mathbf{x}_i)
\right) + \xi_i \tag{7}\]

\[\dot{w}_{ij} = -\gamma w_{ij} + \psi(\mathbf{x}_i,\mathbf{x}_j)
\mathbb{I}_{ij} \tag{8}\]

\[\dot{S}_t = -\mu S_t + \beta\sum_{i,j}w_{ij}\,\operatorname{tr}\!
\left(\mathbf{x}_i\mathbf{x}_j^{\top}\right) \tag{9}\]

Table~\ref{tbl-instantiations} summarizes the FCMS continuous-time
instantiations.

\begin{longtable}[]{@{}
  >{\raggedright\arraybackslash}p{(\linewidth - 4\tabcolsep) * \real{0.3333}}
  >{\raggedright\arraybackslash}p{(\linewidth - 4\tabcolsep) * \real{0.3333}}
  >{\raggedright\arraybackslash}p{(\linewidth - 4\tabcolsep) * \real{0.3333}}@{}}
\caption{FCMS Continuous-Time
Instantiations}\label{tbl-instantiations}\tabularnewline
\toprule\noalign{}
\begin{minipage}[b]{\linewidth}\raggedright
FCMS Operator
\end{minipage} & \begin{minipage}[b]{\linewidth}\raggedright
Continuous Instantiation
\end{minipage} & \begin{minipage}[b]{\linewidth}\raggedright
Source
\end{minipage} \\
\midrule\noalign{}
\endfirsthead
\toprule\noalign{}
\begin{minipage}[b]{\linewidth}\raggedright
FCMS Operator
\end{minipage} & \begin{minipage}[b]{\linewidth}\raggedright
Continuous Instantiation
\end{minipage} & \begin{minipage}[b]{\linewidth}\raggedright
Source
\end{minipage} \\
\midrule\noalign{}
\endhead
\bottomrule\noalign{}
\endlastfoot
\(f_i(\mathbf{x}_{i,t}, G_{i,t}, S_t)\) &
\(\dot{\mathbf{x}}_i = -\eta_i \nabla_{\mathbf{x}_i} L_i + \xi_i\) & MBI
Theorem 4 \citep{grassi2025mbi} \\
\(\Psi(S_t, \mathbf{x}_t)\) &
\(\dot{w}_{ij} = -\gamma w_{ij} + \psi(\mathbf{x}_i,\mathbf{x}_j)\,\mathbb{I}_{ij}\)
& Adapted from CMGP \citep{sarkar2025memory} \\
\(\Phi(L_t^{\text{global}}, \mathbf{x}_t, S_t)\) &
\(\mathbf{G}_t = \alpha_1 S_t \mathbf{x}_i + \alpha_2 \mathbf{L}_t \mathbf{x}_i\)
& FCMS Axioms 1--3 \citep{grassi2026fcms} \\
\(\mathcal{A}(\mathbf{x}_t, S_t)\) &
\(\beta \sum_{i,j} w_{ij}\,\mathbf{x}_i\mathbf{x}_j^{\top}\) & FCMS +
CMGP \citep{grassi2026fcms, sarkar2025memory} \\
\(T(\mathbf{x}_t, S_t)\) & Full coupled ODE system & This paper \\
\end{longtable}

Each row identifies the abstract FCMS operator, its continuous-time
instantiation derived in Section~\ref{sec-agent-continuous-time-lifting}
and Section~\ref{sec-network-layer-insertion}, and the source framework
providing the functional form. The final row, \(T(\mathbf{x}_t, S_t)\),
denotes the complete closed-loop transition operator whose stability
properties are analyzed in Section~\ref{sec-stability-analysis}.

\subsection{Architectural
Interpretation}\label{sec-architectural-interpretation}

The three-layer architecture admits the following interpretation: (i)
the agent layer reflects bounded rational decision-making under
incentive alignment: each agent follows a gradient flow shaped by the
DPM incentive signal, ensuring local optimization remains compatible
with the global objective. (ii) The network layer captures institutional
and relational memory: edge weights evolve as a function of behavioral
proximity, recording the history of agent interactions as a decaying
imprint on the graph topology. (iii) The environmental layer encodes
macro-level coordination pressure: the scalar state \(S_t\) accumulates
the collective topological energy of the graph and feeds it back through
the incentive field that shapes future agent decisions.

\section{Stability Analysis and Main
Theorem}\label{sec-stability-analysis}

\subsection{Lyapunov Candidate}\label{sec-lyapunov-candidate}

Define \(V(t) \in \mathbb{R}^+\), the Lyapunov candidate function which
represents the total energetic load across all three layers as

\begin{equation}\phantomsection\label{eq-lyapunov}{V(t) = \frac{1}{2}\sum_i\|\mathbf{x}_i\|^2 + \frac{1}{2}
\|\mathbf{W}_t\|_F^2 + \frac{1}{2}S_t^2 \tag{10}}\end{equation}

where \(\sum_i\|\mathbf{x}_i\|^2\) represents the total agent kinetic
energy, \(\|\mathbf{W}_t\|_F^2\) the total network structural energy,
and \(S_t^2\) the environmental memory energy. The system inherits the
regularity conditions from MBI \citep[Appendix B.1]{grassi2025mbi},
where (i) \(\mathcal{L}^{\text{global}}\) is \(C^2\) and bounded below,
(ii) \(C_i(\mathbf{x}_i)\) is strictly convex, (iii) gradients are
Lipschitz continuous, and (iv) strong convexity constant \(c\) satisfies
\(\mathbf{x}_i^T\nabla_{\mathbf{x}_i}\mathcal{L}_i \geq c\|\mathbf{x}_i\|^2\).
Building on the forward-invariance result of FCMS \citep[Proposition
A.1.1]{grassi2026fcms}, I compute \(\dot{V}(t)\) along trajectories of
the coupled system
Equation~\ref{eq-agent-ode}--Equation~\ref{eq-env-ode}.

\subsection{Preliminary Lemma}\label{sec-preliminary-lemma}

\begin{lemma}[]\protect\hypertarget{lem-gaussian}{}\label{lem-gaussian}

Let \(\psi(\mathbf{x}_i,\mathbf{x}_j)=\exp\!\left(-\frac{\|\mathbf{x}_i-
\mathbf{x}_j\|^2}{2\sigma^2}\right)\). Then for all
\(\mathbf{x}_i,\mathbf{x}_j\in\mathbb{R}^{d}\),

\[0 \leq \psi(\mathbf{x}_i,\mathbf{x}_j) \leq 1\]

and consequently \(\|\Psi(\mathbf{x})\|_F^2 \leq N^2\).

\end{lemma}

\begin{proof}
The Gaussian kernel satisfies \(\psi \geq 0\) by non-negativity of the
exponential and \(\psi \leq 1\) since \(\exp(-z)\leq 1\) for all
\(z\geq 0\). The Frobenius bound follows since \(\Psi(\mathbf{x})\in
\mathbb{R}^{N\times N}\) and each entry satisfies
\(\psi(\mathbf{x}_i,\mathbf{x}_j)\leq 1\), giving

\[\|\Psi(\mathbf{x})\|_F^2 = \sum_{i,j}\psi(\mathbf{x}_i,\mathbf{x}_j)^2
\leq N^2. \qquad\]
\end{proof}

\subsection{\texorpdfstring{Full Derivation of
\(\dot{V}(t)\)}{Full Derivation of \textbackslash dot\{V\}(t)}}\label{sec-full-derivation-v}

Decompose \(\dot{V}(t)\) into three terms as

\[\dot{V} = \underbrace{\sum_i \mathbf{x}_i^{\top}\dot{\mathbf{x}}_i}_{
\text{Term I}} + \underbrace{\langle\mathbf{W}_t,\dot{\mathbf{W}}_t
\rangle_F}_{\text{Term II}} + \underbrace{\langle S_t,\dot{S}_t\rangle}_{
\text{Term III}}.\]

\subsubsection{Term I --- The Agent Layer}\label{sec-term-I}

Substituting Equation~\ref{eq-agent-ode} gives

\[\sum_i \mathbf{x}_i^{\top}\dot{\mathbf{x}}_i = -\eta \sum_i
\mathbf{x}_i^{\top}\nabla_{\mathbf{x}_i}\mathcal{L}_i\!\left(
\mathbf{x}_i,\mathbf{X}_{-i};\Phi\right) + \sum_i
\mathbf{x}_i^{\top}\xi_i.\]

Applying the strong convexity condition from MBI \citep[Appendix
B.1]{grassi2025mbi}:

\[\mathbf{x}_i^{\top}\nabla_{\mathbf{x}_i}\mathcal{L}_i \geq
c\|\mathbf{x}_i\|^2.\]

The noise term \(\sum_i \mathbf{x}_i^{\top}\xi_i\) is bounded in
expectation. Conclude

\[\text{Term I} \leq -\eta c\sum_i\|\mathbf{x}_i\|^2 + \text{noise}.\]

\subsubsection{Term II --- The Network Layer}\label{sec-term-II}

Substituting Equation~\ref{eq-network-ode} and expanding

\[\langle\mathbf{W}_t,\dot{\mathbf{W}}_t\rangle_F = -\gamma\|\mathbf{W}_t
\|_F^2 + \langle\mathbf{W}_t,\Psi(\mathbf{x})\rangle_F.\]

Apply Young's inequality \citep{hardy1952} with \(\epsilon=\gamma\):

\[\langle\mathbf{W}_t,\Psi\rangle_F \leq \frac{\gamma}{2}\|\mathbf{W}_t
\|_F^2 + \frac{1}{2\gamma}\|\Psi\|_F^2.\]

Bounding \(\|\Psi\|_F^2 \leq N^2\) by Lemma~\ref{lem-gaussian}. Conclude

\[\text{Term II} \leq -\frac{\gamma}{2}\|\mathbf{W}_t\|_F^2 +
\frac{N^2}{2\gamma}.\]

\subsubsection{Term III --- Environment Layer}\label{sec-term-III}

Substitute Equation~\ref{eq-env-ode}. Since \(S_t\in\mathbb{R}\) is
scalar, \(\langle S_t,\dot{S}_t\rangle=S_t\dot{S}_t\) and
\(\|S_t\|^2=S_t^2\),

\[\langle S_t,\dot{S}_t\rangle = S_t\dot{S}_t.\]

Expanding:

\[S_t\dot{S}_t = -\mu S_t^2 + \beta S_t\sum_{i,j}w_{ij}\operatorname{tr}
\!\left(\mathbf{x}_i\mathbf{x}_j^{\top}\right).\]

Let \(G_{ij} = \operatorname{tr}\!\left(\mathbf{x}_i\mathbf{x}_j^{\top}
\right) = \mathbf{x}_i^{\top}\mathbf{x}_j\), so that
\(\sum_{i,j}w_{ij}\operatorname{tr}\!\left(\mathbf{x}_i\mathbf{x}_j^{\top}
\right) = \langle\mathbf{W}_t, G\rangle_F\). Applying the Frobenius
Cauchy--Schwarz inequality:

\[\left|\sum_{i,j}w_{ij}\operatorname{tr}\!\left(\mathbf{x}_i
\mathbf{x}_j^{\top}\right)\right| \leq \|\mathbf{W}_t\|_F\|G\|_F.\]

Since

\[\|G\|_F^2 = \sum_{i,j}\left(\mathbf{x}_i^{\top}\mathbf{x}_j\right)^2
\leq \sum_{i,j}\|\mathbf{x}_i\|^2\|\mathbf{x}_j\|^2 = \left(\sum_i
\|\mathbf{x}_i\|^2\right)^2,\]

it follows that \(\|G\|_F \leq \sum_i\|\mathbf{x}_i\|^2\), and therefore

\[\left|\beta S_t\sum_{i,j}w_{ij}\operatorname{tr}\!\left(\mathbf{x}_i
\mathbf{x}_j^{\top}\right)\right| \leq \beta|S_t|\|\mathbf{W}_t\|_F
\left(\sum_i\|\mathbf{x}_i\|^2\right).\]

Applying Young's inequality \citep{hardy1952} with \(a=|S_t|\) and
\(b=\beta\|\mathbf{W}_t\|_F\sum_i\|\mathbf{x}_i\|^2\):

\[\beta S_t\sum_{i,j}w_{ij}\operatorname{tr}\!\left(\mathbf{x}_i
\mathbf{x}_j^{\top}\right) \leq \frac{\mu}{2}S_t^2 + \frac{\beta^2}
{2\mu}\|\mathbf{W}_t\|_F^2\sum_i\|\mathbf{x}_i\|^2.\]

From Equation~\ref{eq-network-ode}, edge weights satisfy
\(w_{ij}(t)\leq 1/\gamma\) at steady state, since \(\dot{w}_{ij}=0\)
with \(\psi\leq 1\) implies
\(w_{ij}^{*} = \psi(\mathbf{x}_i,\mathbf{x}_j)\mathbb{I}_{ij}/\gamma
\leq 1/\gamma\). Therefore \(\|\mathbf{W}_t\|_F^2 \leq N^2/\gamma^2\).
Substituting into the Term III conclusion:

\[\text{Term III} \leq -\frac{\mu}{2}S_t^2 + \frac{\beta^2 N^2}
{2\mu\gamma^2}\sum_i\|\mathbf{x}_i\|^2.\]

\subsubsection{Collect All Terms}\label{sec-all-terms}

Combining Terms I--III yields

\[\dot{V} \leq -\left(\eta c - \frac{\beta^2 N^2}{2\mu\gamma^2}\right)
\sum_i\|\mathbf{x}_i\|^2 - \frac{\gamma}{2}\|\mathbf{W}_t\|_F^2 -
\frac{\mu}{2}S_t^2 + C.\]

For the coefficient of \(\sum_i\|\mathbf{x}_i\|^2\) to remain negative,
it suffices that \(\eta c > \frac{\beta^2 N^2}{2\mu\gamma^2}\), which
after rearrangement and specialization to \(N=2\), \(c=1\) yields

\[4\beta^2 < 2\eta\mu\gamma^2.\]

This specialization is consistent with the minimal mean-field
specification of FCMS \citep[Appendix A.5]{grassi2026fcms}.

\subsection{Main Theorem}\label{sec-main-theorem}

\begin{theorem}[]\protect\hypertarget{thm-dissipativity}{}\label{thm-dissipativity}

Under regularity conditions inherited from MBI \citep[Assumptions
B.1.2]{grassi2025mbi} and the forward-invariance condition of FCMS
\citep[Proposition A.1.1]{grassi2026fcms}, the continuous-time system
Equation~\ref{eq-agent-ode}--Equation~\ref{eq-env-ode} is globally
dissipative and all trajectories remain bounded within a
forward-invariant region if

\[4\beta^2 < 2\eta\mu\gamma^2.\]

\end{theorem}

\begin{proof}
Under the stability condition \(4\beta^2 < 2\eta\mu\gamma^2\), the
coefficient of each term in \(\dot{V}\) is negative. Standard Lyapunov
comparison arguments \citep[Theorem 4.18]{khalil2002} then yield
\(\dot{V} \leq -\delta V + C\) for some \(\delta > 0\) and bounded
remainder \(C\), implying ultimate boundedness.
\end{proof}

\begin{refremark}
This condition generalizes the discrete-time FCMS stability condition
\(4\eta\beta^2 < \gamma\) \citep{grassi2026fcms} by incorporating the
environmental decay rate \(\mu\) and the edge decay rate \(\gamma\) into
the continuous-time dissipation product, yielding \(4\beta^2 <
2\eta\mu\gamma^2\). Notably, \(\gamma\) now appears squared on the
right-hand side, meaning edge memory decay contributes quadratically to
stability in the continuous-time system. The general \(N\)-agent
stability condition \(\beta^2 < \frac{2\eta c\mu\gamma^2}{N^2}\)
established in the derivation above suggests a corresponding Hopf
bifurcation threshold for arbitrary population sizes; formal
verification is left for future work. The structural stability of the
discrete-time analogue under nonlinear perturbations is established in
FCMS \citep[Appendix A.7]{grassi2026fcms}; the continuous-time analogue
follows from standard hyperbolic fixed-point theory \citep[Theorem
4.18]{khalil2002}.

\label{rem-generalization}

\end{refremark}

\subsection{Bifurcation
Correspondence}\label{sec-bifurcation-correspondence}

The stability thresholds of CMGP and FCMS share a common structure:
stability is maintained when dissipation dominates feedback
amplification. Table~\ref{tbl-bifurcation} summarizes the stability
conditions across the three related systems:

\begin{longtable}[]{@{}
  >{\raggedright\arraybackslash}p{(\linewidth - 4\tabcolsep) * \real{0.3333}}
  >{\raggedright\arraybackslash}p{(\linewidth - 4\tabcolsep) * \real{0.3333}}
  >{\raggedright\arraybackslash}p{(\linewidth - 4\tabcolsep) * \real{0.3333}}@{}}
\caption{Bifurcation Correspondence Across
Systems}\label{tbl-bifurcation}\tabularnewline
\toprule\noalign{}
\begin{minipage}[b]{\linewidth}\raggedright
System
\end{minipage} & \begin{minipage}[b]{\linewidth}\raggedright
Stable Condition
\end{minipage} & \begin{minipage}[b]{\linewidth}\raggedright
Unstable When
\end{minipage} \\
\midrule\noalign{}
\endfirsthead
\toprule\noalign{}
\begin{minipage}[b]{\linewidth}\raggedright
System
\end{minipage} & \begin{minipage}[b]{\linewidth}\raggedright
Stable Condition
\end{minipage} & \begin{minipage}[b]{\linewidth}\raggedright
Unstable When
\end{minipage} \\
\midrule\noalign{}
\endhead
\bottomrule\noalign{}
\endlastfoot
Continuous-Time CMGP \citep{sarkar2025memory} &
\(\alpha_s > \alpha_c = 1/K\) & Feedback curvature exceeds memory
decay \\
Discrete-Time FCMS \citep{grassi2026fcms} & \(4\eta\beta^2 < \gamma\) &
Coupling gain exceeds dissipation \\
Continuous-Time FCMS (this paper) & \(4\beta^2 < 2\eta\mu\gamma^2\) &
Feedback gain exceeds dissipation product \\
\end{longtable}

Across all three systems, the underlying principle is the same: memory
dissipation must outpace feedback gain in order to maintain coherent
bounded dynamics. Stability requires that stored coordination history
dissipates fast enough to prevent feedback amplification.

When \(4\beta^2 > 2\eta\mu\gamma^2\), the dissipation channels can no
longer offset feedback-driven amplification. As the threshold is
approached, the system exhibits critical slowing down and increased
state variance. Beyond the threshold, the fixed point loses stability
through a Hopf bifurcation in the continuous-time system
\citep{guckenheimer1983}, generating a limit cycle. This is the
continuous-time analogue of the Neimark-Sacker bifurcation identified in
the discrete FCMS \citetext{\citealp[Appendix
B.5]{grassi2026fcms}; \citealp{kuznetsov2004}}. The structural parallel
with CMGP's threshold \(\alpha_c = 1/K\) confirms that
memory-dissipation competition is a universal organizing principle
across physical and strategic multi-agent systems.

\section{A Minimal Theoretical Simulation}\label{sec-simulation}

\subsection{Simulation Setup}\label{sec-setup}

Numerical simulations are performed with \(N=2\) agents with scalar
state \(d=1\), consistent with the minimal mean-field specification of
FCMS \citep[Appendix A.5]{grassi2026fcms} under which the clean
stability condition \(4\beta^2 < 2\eta\mu\gamma^2\) is derived. Agents
are initialized at fixed symmetric positions \(\mathbf{x}_1(0)=2\),
\(\mathbf{x}_2(0)=-2\), establishing an initial coordination gap
\(d_0=x_1-x_2=4\). The coupled system
Equation~\ref{eq-agent-ode}--Equation~\ref{eq-env-ode} is integrated
using a fourth-order Runge-Kutta scheme with step size
\(\Delta t=0.005\) over time horizon \(T=10\). Two parameter regimes are
compared, distinguished solely by the coupling gain \(\beta\), with all
other parameters held fixed, isolating the effect of feedback strength
on long-run stability. Table~\ref{tbl-parameters} summarizes all
parameter choices.

\begin{longtable}[]{@{}llll@{}}
\caption{Simulation Parameter
Choices}\label{tbl-parameters}\tabularnewline
\toprule\noalign{}
Parameter & Stable & Unstable & Role \\
\midrule\noalign{}
\endfirsthead
\toprule\noalign{}
Parameter & Stable & Unstable & Role \\
\midrule\noalign{}
\endhead
\bottomrule\noalign{}
\endlastfoot
\(\eta\) & 0.5 & 0.5 & Agent learning rate \\
\(\beta\) & 0.1 & 3.0 & Coupling gain \\
\(\gamma\) & 2.0 & 2.0 & Edge decay rate \\
\(\mu\) & 2.0 & 2.0 & Environmental decay rate \\
\(\sigma\) & 2.0 & 2.0 & Gaussian kernel width \\
\(\alpha_1,\alpha_2\) & 1.0 & 1.0 & Incentive field parameters \\
\(4\beta^2\) & 0.04 & 36.0 & Feedback gain \\
\(2\eta\mu\gamma^2\) & 8.0 & 8.0 & Dissipation product \\
Condition \(4\beta^2 < 2\eta\mu\gamma^2\) & \(\checkmark\) & \(\times\)
& \\
\end{longtable}

Large-scale numerical validation with \(N=10^6\) agents confirming
mean-field convergence consistent with FCMS \citep[Appendix
B.7]{grassi2026fcms} is available in the accompanying repository at
\texttt{github.com/stevefatz95/fcms-continuous}.

\subsection{Stable Regime}\label{sec-stable-regime}

Figure~\ref{fig-vt} shows that under the stable parameterization
\(V(t)\) decreases monotonically from \(V(0)=4.125\) to \(V(T)=0.25\), a
reduction by a factor of approximately sixteen over the simulation
horizon. The decay is smooth and consistent with the ultimate
boundedness condition \(\dot{V}\leq-\delta V+C\) established in
Theorem~\ref{thm-dissipativity}. The dissipation channels \(\gamma\) and
\(\mu\) jointly suppress the Lyapunov energy across all three layers
simultaneously, as the three-term decomposition of
Section~\ref{sec-full-derivation-v} predicts. Figure~\ref{fig-edges}
(left panel) reveals the network-level mechanism underlying this
convergence. Edge weights begin near zero at early time \(t=0.25\) and
concentrate near \(w_{ij}^{*}\approx0.5\) at \(t=T\), consistent with
the steady-state prediction
\(w_{ij}^{*}=\psi(\mathbf{x}_i,\mathbf{x}_j)/\gamma\) from
Equation~\ref{eq-network-ode}. As the incentive field \(\mathbf{G}_t\)
aligns agent behavior under the DPM, pairwise distances
\(|\mathbf{x}_i-\mathbf{x}_j|\) decrease, the Gaussian kernel
\(\psi(\mathbf{x}_i,\mathbf{x}_j)\) strengthens, and the network builds
durable coordination structure. Edge weight concentration thus records
the collapse of initial agent dispersion into a stable network
structure.

\begin{figure}[H]

\centering{

\includegraphics[width=0.9\linewidth,height=\textheight,keepaspectratio]{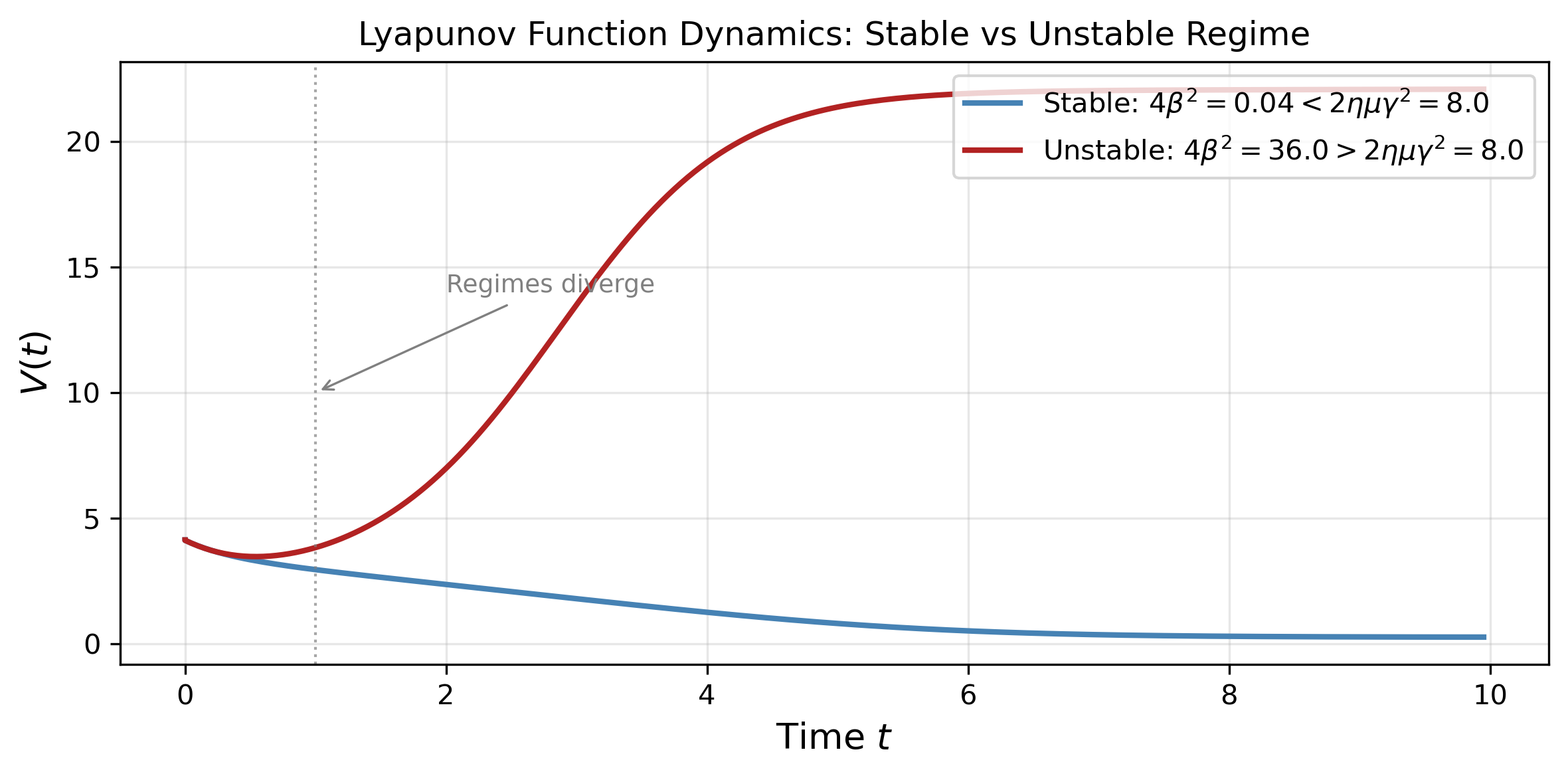}

}

\caption{\label{fig-vt}Lyapunov function \(V(t)\) under stable (blue,
\(4\beta^2=0.04<2\eta\mu\gamma^2=8.0\)) and unstable (red,
\(4\beta^2=36.0>2\eta\mu\gamma^2=8.0\)) parameterizations with \(N=2\)
agents. The stable regime converges monotonically from \(V(0)=4.125\) to
\(V(T)=0.25\), consistent with Theorem~\ref{thm-dissipativity}. The
unstable regime rises to \(V(T)=22.09\) and plateaus in an elevated
oscillatory state, consistent with the predicted Hopf instability
\citep{guckenheimer1983}.}

\end{figure}%

\subsection{Unstable Regime}\label{sec-unstable-regime}

When the stability condition is violated, \(V(t)\) rises from
\(V(0)=4.125\) to \(V(T)=22.09\) and remains elevated throughout the
simulation horizon, as shown in Figure~\ref{fig-vt}. The system never
enters a dissipative phase. Instead the feedback loop amplifies
perturbations, producing sustained energy growth consistent with the
loss of fixed-point stability predicted by the Hopf bifurcation analysis
of Section~\ref{sec-bifurcation-correspondence}. Figure~\ref{fig-edges}
(right panel) reveals the structural mechanism underlying this
instability. Edge weights begin forming at early time but collapse
toward zero by \(t=1.5\), as agents diverge under amplified feedback. As
pairwise distances \(|\mathbf{x}_i-\mathbf{x}_j|\) grow, the Gaussian
kernel \(\psi(\mathbf{x}_i,\mathbf{x}_j)\rightarrow0\), starving the
environmental accumulation term \(\dot{S}_t\) of input and collapsing
the coordination structure that sustains the incentive field. This
self-reinforcing cascade --- diverging agents, decaying edges, weakening
incentives --- is the continuous-time analogue of the coordination
breakdown identified numerically in FCMS \citep[Appendix
B.5]{grassi2026fcms}. The qualitative correspondence between the
continuous-time simulation and the discrete FCMS numerical analysis
confirms that the dissipative-feedback mechanism governing coordination
is structurally stable across discrete and continuous time, consistent
with the structural stability of smooth perturbations of hyperbolic
fixed points \citep[Theorem 4.18]{khalil2002}. A full numerical
characterization of the Hopf bifurcation in the continuous-time system,
including the critical threshold curve in parameter space, is left for
future work.

\begin{figure}[H]

\centering{

\includegraphics[width=1\linewidth,height=\textheight,keepaspectratio]{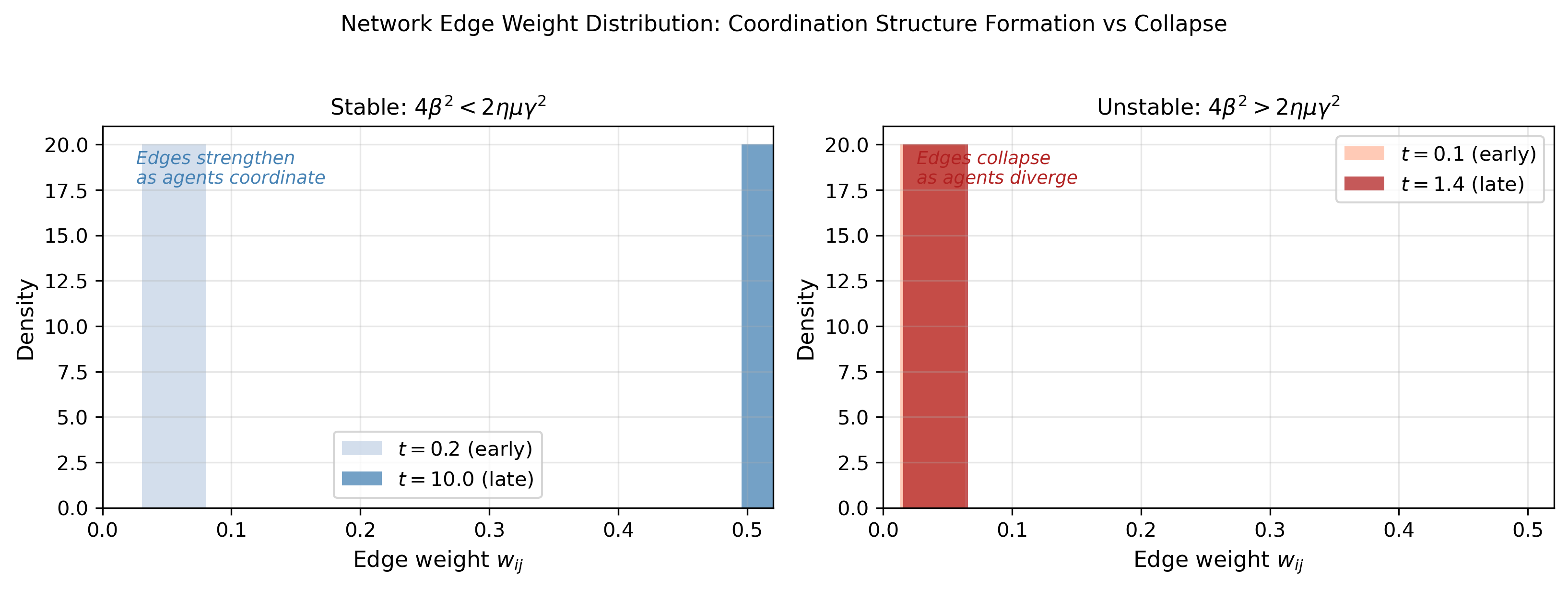}

}

\caption{\label{fig-edges}Network edge weight distribution \(w_{ij}\) at
early time (light) and late time (dark) under stable (left,
\(4\beta^2<2\eta\mu\gamma^2\)) and unstable (right,
\(4\beta^2>2\eta\mu\gamma^2\)) regimes. Under the stable
parameterization edge weights strengthen and concentrate near
\(w_{ij}^{*}\approx\psi/\gamma\), reflecting network coordination
structure formation. Under the unstable parameterization edge weights
collapse toward zero as agents diverge, reflecting coordination
structure failure. The contrast confirms the network-level mechanism
underlying Theorem~\ref{thm-dissipativity}.}

\end{figure}%

\section{Discussion}\label{sec-discussion}

The continuous-time realization of FCMS demonstrates that none of the
single component frameworks could establish the joint stability result
alone. MBI \citep{grassi2025mbi} lacks environmental memory and operates
in discrete iterations over a fixed D-DAG. CMGP \citep{sarkar2025memory}
focuses on a single physical particle without strategic agency or
incentive alignment. The discrete FCMS \citep{grassi2026fcms}, while
laying the architectural groundwork, cannot prove continuous-time
stability. The unified continuous-time system provides a single
computable stability criterion \(4\beta^2 < 2\eta\mu\gamma^2\) governing
coordination across all three layers simultaneously. The threshold is
observable and computable from system parameters, making it a diagnostic
tool for coordination breakdown in social, economic, and multi-agent AI
systems.

The continuous-time FCMS is the macro-scale strategic analogue of
Sarkar's memory engine \citep{sarkar2025memory}, which shows coherence
emerging from coupling in a single physical particle. This paper shows
the same principle operating in a population of \(N\) strategic agents.
The key distinction is that in Sarkar's framework coherence is physical
and driven by gradient forcing alone, while here coherence is economic
and requires explicit incentive alignment via the MBI DPM
\citep{grassi2025mbi}. This work extends CMGP's bifurcation threshold
\(\alpha_s > \alpha_c = 1/K\) \citep{sarkar2025memory} to the
multi-agent strategic stability criterion
\(4\beta^2 < 2\eta\mu\gamma^2\), revealing an analogous universal
organizing principle across physics and economics: memory dissipation
must outpace feedback gain. Sarkar's memory engine principle is
therefore not limited to physical substrates but extends to any system
where agents interact through a persistent environment.

Several limitations remain. First, the incentive field specification
\(\Phi = \alpha_1 S_t \mathbf{x}_i + \alpha_2 \mathbf{L}_t \mathbf{x}_i\)
is the minimal linear form satisfying FCMS Axioms 1--3 \citep[Section
2.3]{grassi2026fcms}, chosen for analytical tractability; nonlinear
specifications may produce more complex dynamics. Second, the learning
rate \(\eta\) is shared homogeneously across agents; heterogeneous
adaptation rates are left for future analysis. Third, the environmental
state \(S_t \in \mathbb{R}\) is retained as a scalar; the matrix-valued
generalization \(S_t \in \mathbb{R}^{d\times d}\) may represent more
structured institutional memory. Fourth, the clean stability condition
is proved for the minimal \(N=2\) mean-field specification
\citep[Appendix A.5]{grassi2026fcms}; the general \(N\)-agent Hopf
bifurcation structure is conjectured and left for formal verification.
Fifth, no empirical calibration has been provided.

Future work should address four directions. First, extending the
framework to heterogeneous learning rates \(\eta_i\), connecting to
Sarkar's heterogeneous diffusion extension \citep{sarkar2025cmgp} as the
physical parallel. Second, characterizing the class of admissible
nonlinear incentive distribution operators \(\Phi\) beyond the linear
specification. Third, completing the Hopf bifurcation characterization
\citep{guckenheimer1983} for the general \(N\)-agent continuous-time
system including the critical threshold curve in parameter space.
Fourth, empirical calibration connecting parameters \(\eta\), \(\beta\),
\(\gamma\), and \(\mu\) to observable economic or institutional
variables, using the stability threshold \(4\beta^2 < 2\eta\mu\gamma^2\)
as an early warning indicator for coordination breakdown.

\section{Conclusion}\label{sec-conclusion}

This paper develops the continuous-time realization of FCMS
\citep{grassi2026fcms} by instantiating the abstract operators \(f_i\)
and \(\Psi\) via an agent layer defined by MBI \citep{grassi2025mbi} and
a network layer adapted from CMGP \citep{sarkar2025memory} respectively.
The resulting system is globally dissipative and trajectories remain
bounded if the stability condition \(4\beta^2 < 2\eta\mu\gamma^2\) is
satisfied. More specifically, \(\eta\) governs how strongly agents
respond to incentives, \(\beta\) governs how strongly the network feeds
back into the environment, \(\gamma\) governs how quickly edge memory
decays, and \(\mu\) governs how quickly environmental memory decays.
Stability requires the product of the two dissipation channels
\(2\eta\mu\gamma^2\) to outpace the amplified feedback gain
\(4\beta^2\). When the stability condition is violated the system
undergoes a Hopf bifurcation \citep{guckenheimer1983}, the coordination
manifold loses stability, and the self-reinforcing cascade of diverging
agents, decaying edges, and weakening incentives produces coordination
failure. This generalizes both the discrete FCMS condition
\(4\eta\beta^2 < \gamma\) \citep{grassi2026fcms} and CMGP's physical
threshold \(\alpha_c = 1/K\) \citep{sarkar2025memory}, confirming that
memory dissipation must outpace feedback gain as a universal organizing
principle across physical and strategic multi-agent systems. MBI
established the micro-layer, FCMS established the macro-layer
architecture in discrete time, and CMGP established the meso-layer in
continuous physical space. This paper completes the picture by providing
the continuous-time realization that connects all three. The stability
condition \(4\beta^2 < 2\eta\mu\gamma^2\) is not only a theoretical
boundary but a computable early warning criterion \citep{scheffer2009},
enabling empirical validation in economic and institutional systems
where coordination breakdown is observable. In this sense,
\(4\beta^2 < 2\eta\mu\gamma^2\) is the dynamical invisible hand
\citep{smith1776}: not a metaphor for market efficiency, but a
computable structural criterion under which decentralized agents,
responding only to local incentive signals shaped by a persistent
environment and an evolving network, spontaneously achieve collective
order without design.

\section{References}\label{references}

\renewcommand{\bibsection}{}
\bibliography{refs.bib}

@article{hayek1945,
  author    = {Hayek, Friedrich A.},
  title     = {The Use of Knowledge in Society},
  journal   = {American Economic Review},
  year      = {1945},
  volume    = {35},
  number    = {4},
  pages     = {519--530}
}

@article{grassi2025mbi,
  author    = {Grassi, Stefano},
  title     = {Mechanism-Based Intelligence: Differentiable Incentives
               for Rational Coordination and Guaranteed Alignment in
               Multi-Agent Systems},
  journal   = {arXiv preprint},
  year      = {2025},
  eprint    = {2512.20688},
  archivePrefix = {arXiv},
  primaryClass  = {cs.GT},
  url       = {https://arxiv.org/abs/2512.20688}
}

@article{grassi2026fcms,
  author    = {Grassi, Stefano},
  title     = {Feedback-Coupled Memory Systems: A Dynamical Model for
               Adaptive Coordination},
  journal   = {arXiv preprint},
  year      = {2026},
  eprint    = {2603.11560},
  archivePrefix = {arXiv},
  primaryClass  = {cs.MA},
  url       = {https://arxiv.org/abs/2603.11560}
}

@article{sarkar2025cmgp,
  author    = {Sarkar, Aranyak},
  title     = {Non-{M}arkovian {R}oute to {C}oherence in {H}eterogeneous
              {D}iffusive {S}ystems},
  journal   = {Physical Review E},
  volume    = {112},
  number    = {5},
  pages     = {054117},
  year      = {2025},
  month     = {nov},
  publisher = {American Physical Society},
  doi       = {10.1103/6r83-n97h},
  url       = {https://link.aps.org/doi/10.1103/6r83-n97h}
}

@article{sarkar2025memory,
  author    = {Sarkar, Aranyak},
  title     = {Memory Engine: Self-Organized Coherence from Internal Feedback},
  journal   = {Physical Review E},
  volume    = {112},
  number    = {5},
  pages     = {054111},
  year      = {2025},
  month     = {nov},
  publisher = {American Physical Society},
  doi       = {10.1103/t7nk-4p57},
  url       = {https://link.aps.org/doi/10.1103/t7nk-4p57}
}

@book{khalil2002,
  author    = {Khalil, Hassan K.},
  title     = {Nonlinear Systems},
  edition   = {3rd},
  publisher = {Prentice Hall},
  address   = {Upper Saddle River, NJ},
  year      = {2002}
}

@book{hardy1952,
  author    = {Hardy, Godfrey H. and Littlewood, John E. and
               P{\'o}lya, George},
  title     = {Inequalities},
  edition   = {2nd},
  publisher = {Cambridge University Press},
  address   = {Cambridge},
  year      = {1952}
}

@book{kuznetsov2004,
  author    = {Kuznetsov, Yuri A.},
  title     = {Elements of Applied Bifurcation Theory},
  edition   = {3rd},
  publisher = {Springer},
  address   = {New York},
  year      = {2004},
  series    = {Applied Mathematical Sciences},
  volume    = {112}
}

@book{guckenheimer1983,
  author    = {Guckenheimer, John and Holmes, Philip},
  title     = {Nonlinear Oscillations, Dynamical Systems, and
               Bifurcations of Vector Fields},
  publisher = {Springer},
  address   = {New York},
  year      = {1983},
  series    = {Applied Mathematical Sciences},
  volume    = {42}
}

@book{hurwicz2006,
  author    = {Hurwicz, Leonid and Reiter, Stanley},
  title     = {Designing Economic Mechanisms},
  publisher = {Cambridge University Press},
  address   = {Cambridge},
  year      = {2006}
}

@article{vickrey1961,
  author    = {Vickrey, William},
  title     = {Counterspeculation, Auctions, and Competitive Sealed
               Tenders},
  journal   = {Journal of Finance},
  year      = {1961},
  volume    = {16},
  number    = {1},
  pages     = {8--37},
  doi       = {10.2307/2977633}
}

@article{clarke1971,
  author    = {Clarke, Edward H.},
  title     = {Multipart Pricing of Public Goods},
  journal   = {Public Choice},
  year      = {1971},
  volume    = {11},
  number    = {1},
  pages     = {17--33},
  doi       = {10.1007/BF01726210}
}

@article{groves1973,
  author    = {Groves, Theodore},
  title     = {Incentives in Teams},
  journal   = {Econometrica},
  year      = {1973},
  volume    = {41},
  number    = {4},
  pages     = {617--631},
  doi       = {10.2307/1914085}
}

@article{myerson1981,
  author    = {Myerson, Roger B.},
  title     = {Optimal Auction Design},
  journal   = {Mathematics of Operations Research},
  year      = {1981},
  volume    = {6},
  number    = {1},
  pages     = {58--73},
  doi       = {10.1287/moor.6.1.58}
}

@article{olfati2007,
  author    = {Olfati-Saber, Reza and Fax, J. Alex and Murray,
               Richard M.},
  title     = {Consensus and Cooperation in Networked Multi-Agent
               Systems},
  journal   = {Proceedings of the IEEE},
  year      = {2007},
  volume    = {95},
  number    = {1},
  pages     = {215--233},
  doi       = {10.1109/JPROC.2006.887293}
}

@book{newman2010,
  author    = {Newman, Mark E. J.},
  title     = {Networks: An Introduction},
  publisher = {Oxford University Press},
  address   = {Oxford},
  year      = {2010}
}

@book{mesbahi2010,
  author    = {Mesbahi, Mehran and Egerstedt, Magnus},
  title     = {Graph Theoretic Methods in Multiagent Networks},
  publisher = {Princeton University Press},
  address   = {Princeton, NJ},
  year      = {2010}
}

@article{scheffer2009,
  author    = {Scheffer, Marten and Bascompte, Jordi and Brock,
               William A. and Brovkin, Victor and Carpenter, Stephen R.
               and Dakos, Vasilis and Held, Hermann and van Nes, Egbert H.
               and Rietkerk, Max and Sugihara, George},
  title     = {Early-Warning Signals for Critical Transitions},
  journal   = {Nature},
  year      = {2009},
  volume    = {461},
  pages     = {53--59},
  doi       = {10.1038/nature08227}
}

@book{smith1776,
  author    = {Smith, Adam},
  title     = {An Inquiry into the Nature and Causes of the Wealth
               of Nations},
  publisher = {W. Strahan and T. Cadell},
  address   = {London},
  year      = {1776}
}

\section{Appendix}\label{appendix}

\subsection{\texorpdfstring{Appendix A: General \(N\)-Agent Stability
Condition}{Appendix A: General N-Agent Stability Condition}}\label{sec-appendix-a}

The stability condition \(4\beta^2 < 2\eta\mu\gamma^2\) established in
Theorem~\ref{thm-dissipativity} is the specialization to \(N=2\) agents
and convexity constant \(c=1\) of the following general result.

\begin{proposition}[]\protect\hypertarget{prp-general-n}{}\label{prp-general-n}

Under the regularity conditions of Theorem~\ref{thm-dissipativity} with
\(N\) agents and strong convexity constant \(c > 0\), the
continuous-time system
Equation~\ref{eq-agent-ode}--Equation~\ref{eq-env-ode} is globally
dissipative if

\[\beta^2 < \frac{2\eta c\mu\gamma^2}{N^2}.\]

\end{proposition}

\begin{proof}
The \(\dot{V}\) derivation in Section~\ref{sec-full-derivation-v} yields

\[\dot{V} \leq -\left(\eta c - \frac{\beta^2 N^2}{2\mu\gamma^2}\right)
\sum_i\|\mathbf{x}_i\|^2 - \frac{\gamma}{2}\|\mathbf{W}_t\|_F^2 -
\frac{\mu}{2}S_t^2 + C.\]

For the coefficient of \(\sum_i\|\mathbf{x}_i\|^2\) to remain negative
it suffices that

\[\eta c > \frac{\beta^2 N^2}{2\mu\gamma^2}.\]

Rearranging yields the stated condition. Under this condition all three
bracketed coefficients in \(\dot{V}\) are negative and standard Lyapunov
comparison arguments \citep[Theorem 4.18]{khalil2002} yield
\(\dot{V} \leq -\delta V + C\) for some \(\delta > 0\) and bounded
remainder \(C\), implying ultimate boundedness.
\end{proof}

Theorem~\ref{thm-dissipativity} follows by substituting \(N=2\) and
\(c=1\):

\[\beta^2 < \frac{2\eta \cdot 1 \cdot \mu\gamma^2}{4} =
\frac{\eta\mu\gamma^2}{2}\]

which is equivalent to \(4\beta^2 < 2\eta\mu\gamma^2\), consistent with
the minimal mean-field specification of FCMS \citep[Appendix
A.5]{grassi2026fcms}.

\begin{refremark}
The condition \(\beta^2 < \frac{2\eta c\mu\gamma^2}{N^2}\) reveals that
stability becomes harder to maintain as population size grows: for fixed
parameters \(\eta\), \(\mu\), \(\gamma\), the admissible coupling gain
\(\beta\) scales as \(1/N\), reflecting the amplified environmental
feedback generated by larger agent populations. Whether this condition
is tight --- that is, whether instability necessarily follows when it is
violated for general \(N\) --- and the associated Hopf bifurcation
structure for arbitrary population sizes remain open and are left for
future work.

\label{rem-scaling}

\end{refremark}

\subsection{Appendix B: Justification of the Linear Incentive
Field}\label{sec-appendix-b}

The incentive distribution operator adopted throughout this paper is

\begin{equation}\phantomsection\label{eq-phi-spec}{\Phi(S_t, \mathbf{L}_t, \mathbf{x}_i) = \alpha_1 S_t\mathbf{x}_i
+ \alpha_2\mathbf{L}_t\mathbf{x}_i \tag{B.1}}\end{equation}

with \(\alpha_1, \alpha_2 > 0\). FCMS \citep[Section
2.3]{grassi2026fcms} establishes that \(\Phi\) must distribute the
global coordination signal locally --- each agent receives only its own
incentive component without observing \(L_t^{\text{global}}\) directly
--- and must be non-conservative, meaning it cannot be expressed as the
gradient of a scalar functional over \(\mathcal{X}\) alone due to path
dependence induced by \(S_t\). I verify both properties for the linear
specification and add continuity as a standard regularity requirement.

The term \(\alpha_1 S_t\mathbf{x}_i\) couples the scalar environmental
memory \(S_t\) to each agent's own state locally, producing a
directional pressure without requiring observation of the global signal.
The term \(\alpha_2\mathbf{L}_t\mathbf{x}_i\) reduces to
\(\alpha_2\sum_j
w_{ij}(\mathbf{x}_i - \mathbf{x}_j)\) for agent \(i\), which depends
only on the agent's own state and its immediate neighbors --- the
canonical local consensus protocol \citep{olfati2007}. Both terms are
therefore local. Non-conservatism follows because \(S_t\) evolves
endogenously via Equation~\ref{eq-env-ode}, introducing path dependence
in the effective vector field over \(\mathcal{X}\) whenever
\(\beta > 0\), as established in FCMS \citep[Section
2.3]{grassi2026fcms}. Continuity holds since the specification is linear
in each argument. The linear form is the minimal specification
satisfying all three properties; nonlinear extensions satisfying the
same three properties are possible and left for future work.

\subsection{Appendix C: Extended Numerical
Validation}\label{sec-appendix-c}

\subsubsection{C.1 Phase Portrait Analysis}\label{sec-appendix-c1}

Figure~\ref{fig-phase} shows the phase portrait of the coupled system in
the \((S_t, d_t)\) plane, where \(d_t = x_1(t) - x_2(t)\) is the agent
disagreement and \(S_t\) is the environmental memory state. The
coordination equilibrium is located at the origin
\((S^*, d^*) = (0, 0)\), where agent disagreement vanishes and
environmental memory decays to zero.

Under the stable parameterization (\(4\beta^2 = 0.04 <
2\eta\mu\gamma^2 = 8.0\)), the trajectory exhibits spiral convergence
toward the fixed point. Starting from \((S_0, d_0) = (0.5, 4.0)\), the
system winds inward, with the terminal marker approaching the origin
closely by \(t=T\), consistent with the ultimate boundedness established
in Theorem~\ref{thm-dissipativity}. The spiral geometry reflects the
oscillatory transient dynamics as the incentive field and environmental
memory work together to reduce agent disagreement.

Under the unstable parameterization (\(4\beta^2 = 36.0 >
2\eta\mu\gamma^2 = 8.0\)), the trajectory diverges away from the fixed
point. Agent disagreement \(d_t\) grows persistently, driven by the
amplified feedback loop. The phase portrait confirms that the origin
loses stability when the condition \(4\beta^2 < 2\eta\mu\gamma^2\) is
violated, consistent with the Hopf bifurcation analysis of
Section~\ref{sec-bifurcation-correspondence}. The qualitative
distinction between the two panels provides geometric confirmation of
Theorem~\ref{thm-dissipativity} complementing the Lyapunov energy
analysis of Section~\ref{sec-stability-analysis}.

\begin{figure}[H]

\centering{

\includegraphics[width=1\linewidth,height=\textheight,keepaspectratio]{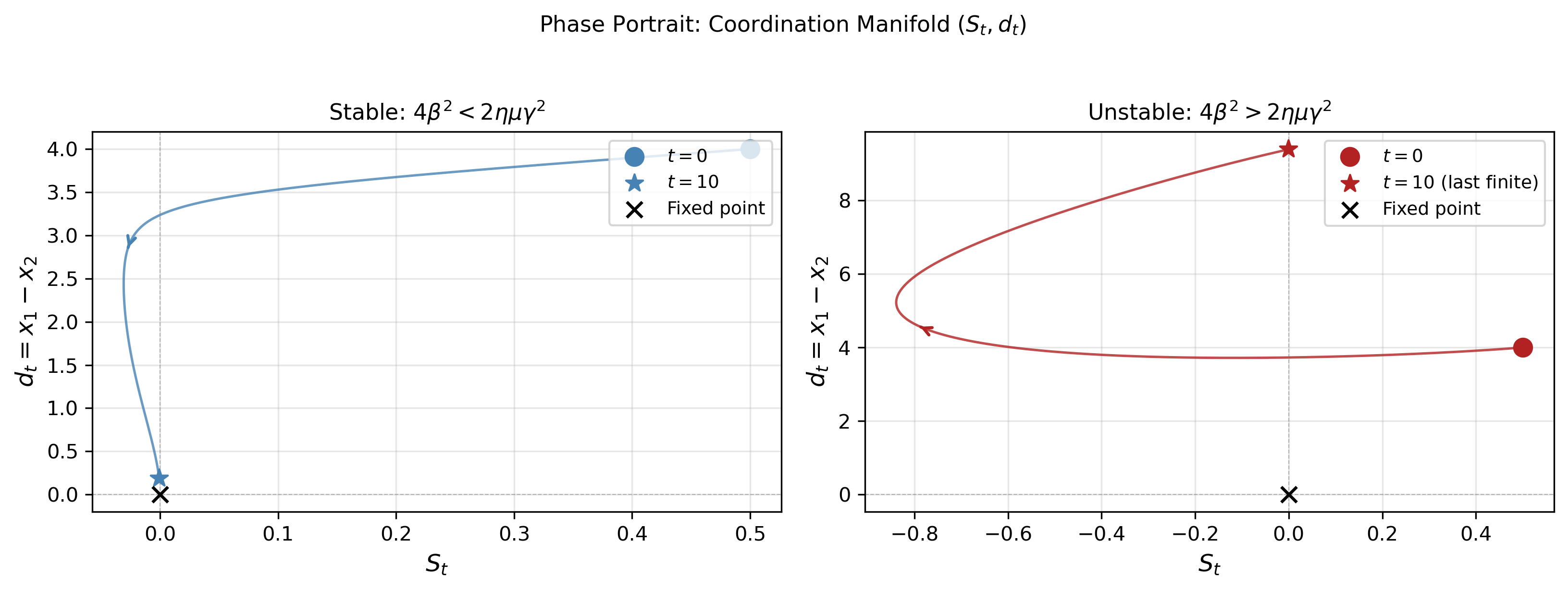}

}

\caption{\label{fig-phase}Phase portrait of the coupled system in the
\((S_t, d_t)\) plane under stable (left, \(4\beta^2<2\eta\mu\gamma^2\))
and unstable (right, \(4\beta^2>2\eta\mu\gamma^2\)) parameterizations.
The circle marks \(t=0\) and the star marks \(t=T\). The fixed point at
the origin \((S^*, d^*)=(0,0)\) is marked with a cross. Under the stable
parameterization the trajectory spirals inward toward the coordination
equilibrium, consistent with Theorem~\ref{thm-dissipativity}. Under the
unstable parameterization agent disagreement \(d_t\) grows persistently,
consistent with the predicted loss of fixed-point stability at the Hopf
bifurcation threshold \citep{guckenheimer1983}.}

\end{figure}%

\subsubsection{C.2 Large-Scale Validation}\label{sec-appendix-c2}

Large-scale numerical validation with \(N=10^6\) agents under the stable
parameterization confirms mean-field convergence, as shown in
Figure~\ref{fig-meanfield}. The mean agent state reaches \(\mu_x(T)=0\),
the population variance decays to zero, and the environmental memory
\(S(T)=0.053\) remains small, consistent with ultimate boundedness under
finite horizon \(T=10\) \citep[Appendix B.7]{grassi2026fcms}. All
simulation code, figure generation scripts, and the full mean-field
validation are available at
\texttt{github.com/stevefatz95/fcms-continuous}.

\begin{figure}[H]

\centering{

\includegraphics[width=0.9\linewidth,height=\textheight,keepaspectratio]{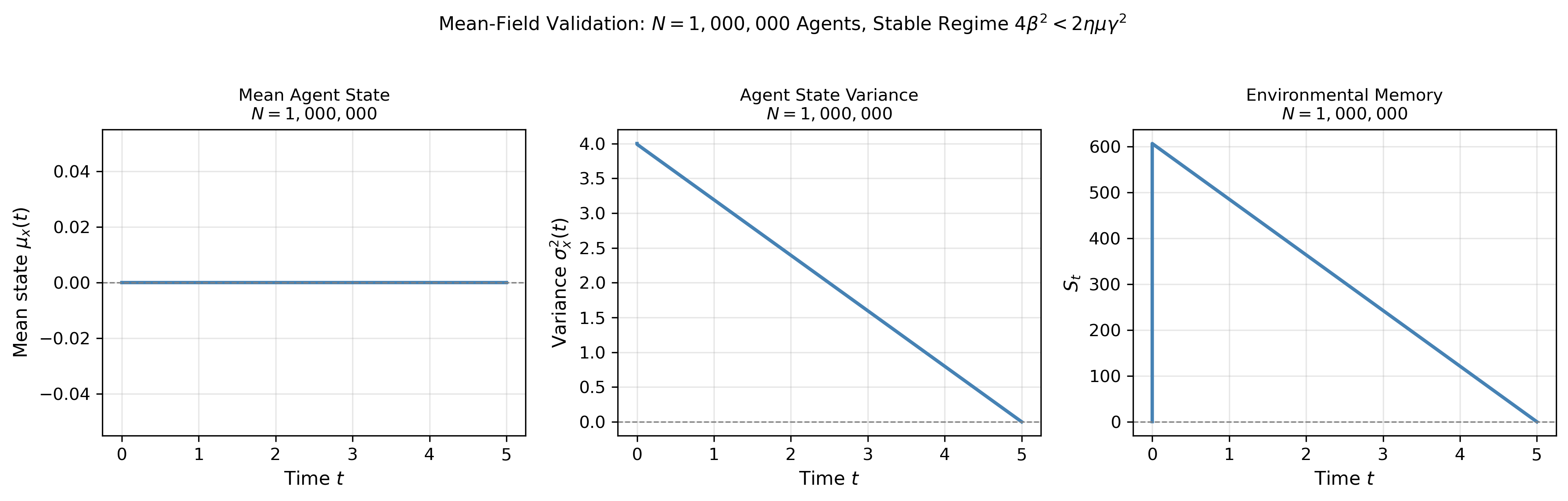}

}

\caption{\label{fig-meanfield}Mean-field validation with \(N=10^6\)
agents under the stable parameterization
(\(4\beta^2=0.04 < 2\eta\mu\gamma^2=8.0\)). The mean agent state
converges to \(\mu_x(T)=0\), the population variance decays to zero, and
the environmental memory reaches \(S(T)=0.053\), all consistent with
ultimate boundedness under Theorem~\ref{thm-dissipativity} over finite
horizon \(T=10\).}

\end{figure}%

\end{document}